\documentclass{article}

\usepackage{iclr2027_conference,times}
\usepackage{latexsym}
\usepackage[T1]{fontenc}
\usepackage[utf8]{inputenc}
\usepackage{microtype}
\usepackage{inconsolata}
\usepackage{graphicx}
\usepackage{amsmath}
\usepackage{amssymb}
\usepackage{amsthm}
\usepackage{bm}
\usepackage{booktabs}
\usepackage{threeparttable}
\usepackage{multirow}
\usepackage{float}
\usepackage{wrapfig}
\usepackage{needspace}
\usepackage{array}
\usepackage{colortbl}
\usepackage{tcolorbox}
\tcbuselibrary{breakable}
\usepackage{fvextra}
\usepackage{capt-of}
\usepackage{hyperref}
\usepackage{url}
\hypersetup{colorlinks=true,linkcolor=blue,citecolor=blue,urlcolor=blue}

\newcommand{\method}{SCAPO}
\newcommand{\methodfull}{Semifactual Credit-Augmented Policy Optimization}
\newcommand{\best}[1]{\textbf{#1}}
\newcommand{\second}[1]{\underline{#1}}
\newcommand{\tokspace}{\raisebox{-0.16ex}{\rule{0.55em}{0.22ex}}}
\definecolor{TableGroup}{HTML}{DCE7FA}
\definecolor{TableAvg}{HTML}{FFE9D6}
\definecolor{PromptBoxHeader}{HTML}{55ADD1}
\definecolor{PromptBoxBody}{HTML}{EAF6FB}
\definecolor{PromptBoxBorder}{HTML}{55ADD1}

\title{Semifactual Credit-Augmented Policy\\Optimization}

\author{%
\textbf{Junshu Pan\textsuperscript{1,2,3}}\quad
\textbf{Zhizhang Fu\textsuperscript{2}}\quad
\textbf{Shulin Huang\textsuperscript{1,2}}\quad
\textbf{Yiran Ding\textsuperscript{2}}\quad
\textbf{Zifan Cheng\textsuperscript{1}}\\
\textbf{Wenqi Shao\textsuperscript{3,4}}\quad
\textbf{Qiaosheng Zhang\textsuperscript{3,4}}\quad
\textbf{Yue Zhang\textsuperscript{2}}\thanks{Corresponding author.}\\
\textsuperscript{1}Zhejiang University\quad
\textsuperscript{2}Westlake University\\
\textsuperscript{3}Shanghai Innovation Institute\quad
\textsuperscript{4}Shanghai AI Laboratory\\
\texttt{\{panjunshu,zhangyue\}@westlake.edu.cn}
}

\iclrfinalcopy

\begin{document}
\maketitle
\fancyhead{}
\lhead{Semifactual Credit-Augmented Policy Optimization}

\begin{abstract}
Reinforcement learning with verifiable rewards (RLVR) has improved the reasoning capabilities of large language models (LLMs), yet their predictions remain sensitive to task-irrelevant prompt features.
We investigate this sensitivity through semifactual prompt interventions that preserve the underlying problem and its answer.
Our analysis reveals substantial variation in token-level sensitivity and shows that suppressing high-drift token candidates during decoding improves reasoning accuracy without updating model weights.
These findings highlight a limitation of Group Relative Policy Optimization (GRPO), which assigns the same outcome-derived advantage to every response token and may reinforce potential spurious dependence alongside useful reasoning.
Motivated by this observation, we introduce \methodfull{} (\method{}), a causally inspired variant of GRPO that incorporates semifactual stability into token-level credit assignment.
\method{} measures token probability drift for fixed responses under semifactual interventions and uses normalized stability scores to reduce advantages for relatively unstable tokens during early training, while granting no additional credit for stability alone.
On Qwen3-4B-Base and Qwen3-1.7B-Base, \method{} improves AIME 2024–2026 accuracy over GRPO by 5.63 and 4.17 percentage points, respectively.
At both model scales, \method{} achieves the best results on most evaluated mathematics benchmarks and all evaluated out-of-distribution benchmarks among the compared methods.
These results suggest that semifactual stability provides an effective training signal for improving reasoning and generalization through finer-grained credit assignment in RLVR.
The code is available at \url{https://github.com/DtYXs/SCAPO}.
\end{abstract}

\begin{figure}[!htb]
\centering
\includegraphics[width=\textwidth]{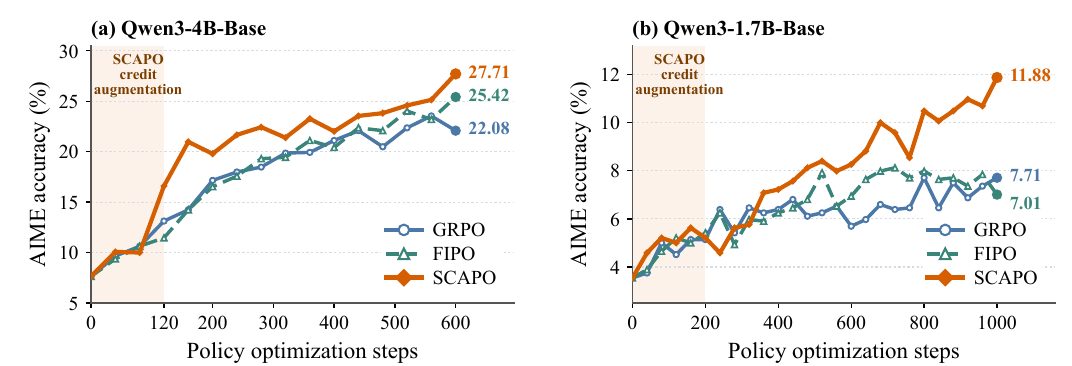}
\caption{\textbf{\method{} achieves the highest AIME accuracy on the two Qwen3 base models.}
Aggregate AIME 2024--2026 accuracy for \method{}, GRPO, and FIPO.
\method{} reaches GRPO's final accuracy in fewer than half as many policy optimization steps and achieves higher final accuracy than baselines at both model scales.
Shaded regions mark SCAPO's semifactual credit augmentation phase.}
\label{fig:grpo-semifact-aime}
\end{figure}

\section{Introduction}

Reinforcement learning with verifiable rewards (RLVR) has advanced the development of large language models (LLMs) with strong reasoning capabilities, such as OpenAI o1~\citep{openai2024reasoning} and DeepSeek-R1~\citep{guo2025deepseek}.
Group Relative Policy Optimization (GRPO)~\citep{shao2024deepseekmath}, a representative RLVR method, optimizes automatically checkable outcomes without process-level supervision.
Recent evidence suggests that RLVR can diminish dependence on spurious correlations and improve generalization under distribution shifts~\citep{fu2026correlation}.
However, robustness evaluations still reveal sensitivity to irrelevant information and input perturbations~\citep{mirzadeh2025gsm,huang2025thinkbench}.
Understanding how training factors affect generalization in RLVR and how to further improve it remains underexplored.

We study spurious dependence in RLVR at the token level.
Building on prior studies of spurious feature reliance~\citep{wang2022identifying,fu2026correlation}, we probe causal invariance~\citep{peters2016causal,arjovsky2019invariant} in LLM inference through semifactual prompt interventions~\citep{goodman1947problem,lu2022rationale}.
Specifically, we alter task-irrelevant prompt features while preserving the mathematical problem and its answer, using drift in the token probabilities of a fixed response as a proxy for potential spurious dependence.
We find that suppressing high-drift token candidates during decoding improves the accuracy of Qwen3-4B-Base~\citep{yang2025qwen3} from 15.8\% to 30.0\% on a 1,000-question mathematical diagnostic panel without updating model weights (see Section~\ref{sec:motivating-observation}).
The above results suggest that LLMs are highly sensitive to token-level spurious features.
However, GRPO assigns the same outcome-derived advantage to every valid token in a response, without directly accounting for this token-level sensitivity.
As a result, positive outcome credit may reinforce potential spurious dependence alongside useful reasoning.

To address the above problem, one intuitive way is to change the RL training process, incorporating semifactual sensitivity into token-level credit assignment.
In this paper, we introduce \textbf{\methodfull{} (\method{})}, a causally inspired variant of GRPO that aims to turn semifactual stability into a token-level credit signal during training.
This signal reveals differences in token sensitivity that final-answer correctness alone cannot distinguish.
Verifiable rewards provide response-level supervision, while semifactual stability refines token-level credit assignment.

In particular, \method{} uses the rollout policy to teacher-force each sampled response under the original prompt and its semifactual perturbations, measuring the probability drifts of each response token.
After aggregating and normalizing these drifts within each prompt group, it adds only the negative part of the resulting stability score to the GRPO advantage as a detached token-level credit augmentation.
Consequently, relatively unstable tokens receive lower advantages, while stability alone earns no additional credit, since stability does not imply correctness.
We apply this credit augmentation during the early phase of training to shape trajectory selection, then continue optimizing the resulting policy with standard GRPO.

\method{} improves AIME 2024--2026 accuracy~\citep{maaAime} over GRPO by \textbf{+5.63} and \textbf{+4.17} points on Qwen3-4B-Base and Qwen3-1.7B-Base~\citep{yang2025qwen3}, respectively.
As shown in Figure~\ref{fig:grpo-semifact-aime}, \method{} reaches GRPO's final AIME accuracy in fewer than half as many policy optimization steps.
Moreover, across both model scales, \method{} outperforms GRPO on all evaluated benchmarks and achieves the best performance on most of the competition-level mathematics benchmarks among all the compared RLVR methods.
\method{} also achieves the highest accuracy on out-of-distribution benchmarks at both model scales.
In addition, our ablations further support the value of aligning credit corrections with semifactual sensitivity and selectively reducing credit for relatively unstable tokens.
These results suggest that semifactual stability complements outcome rewards with an effective signal for finer-grained credit assignment in RLVR.

Our contributions can be summarized as:
\begin{itemize}
    \item We study token-level spurious dependence in LLM inference through semifactual prompt interventions, revealing heterogeneous sensitivity and showing that suppressing high-drift tokens during decoding can improve reasoning accuracy without weight updates.~(Sec.~\ref{sec:motivating-observation})

    \item We introduce \textbf{\method{}}, an approach to token-level credit augmentation in GRPO using detached negative-only corrections derived from fixed-response semifactual stability, without requiring process supervision or an external reward model.~(Sec.~\ref{sec:scapo})

    \item We empirically demonstrate \method{}'s effectiveness on Qwen3-4B-Base and Qwen3-1.7B-Base, achieving the best results on most benchmarks among the compared RLVR methods.
    These results show that semifactual stability can serve as an effective training-time signal for token-level credit assignment in RLVR.~(Sec.~\ref{sec:experiments})
\end{itemize}

\section{Token-Level Semifactual Sensitivity}
\label{sec:motivating-observation}

\begin{figure}[!t]
\centering
\includegraphics[width=\textwidth]{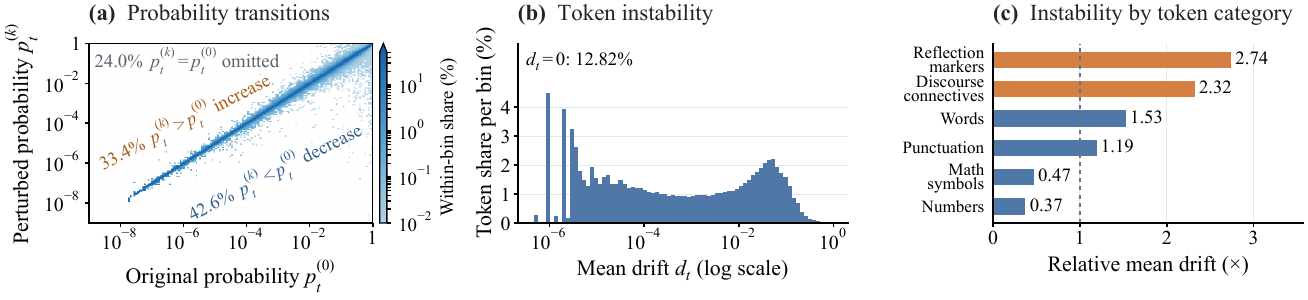}
\caption{\textbf{Semifactual instability varies across tokens.}
(a) Probability transitions under semifactual prompt interventions.
(b) Distribution of mean token drift.
(c) Relative category mean drift normalized by the overall mean, with reflection markers and discourse connectives showing higher sensitivity.
Results use frozen Qwen3-4B-Base before RL.}
\label{fig:motivating-instability}
\end{figure}

We probe potential token-level spurious dependence with semifactual sensitivity.
Using a frozen base model, we characterize how this sensitivity varies across tokens and then examine whether suppressing unstable token candidates improves decoding.

\paragraph{Constructing semifactual perturbations.}
Semifactual perturbations alter incidental features of a prompt while preserving its underlying mathematical problem and final answer~\citep{kenny2021generating}.
Following prior work on reasoning robustness under input perturbations~\citep{mirzadeh2025gsm,huang2025thinkbench}, we construct four types of perturbations: paraphrase, minor typo noise, irrelevant scenario wrapping, and appended irrelevant context.
We use GPT-5.5~\citep{openai2026gpt55} to generate these perturbations, instructing it to preserve quantities, mathematical expressions, constraints, and the requested quantity.

\begin{tcolorbox}[
    colback=PromptBoxBody, colframe=PromptBoxBorder,
    colbacktitle=PromptBoxHeader, coltitle=white,
    title={Example: Paraphrase},
    fonttitle=\small\bfseries, fontupper=\small,
    boxrule=0.5pt, arc=1.5mm, boxsep=0pt,
    left=6pt, right=6pt, top=4pt, bottom=4pt,
    toptitle=3pt, bottomtitle=3pt,
    before skip=6pt, after skip=6pt
]
\textbf{Original:} What is the smallest odd number with four different prime factors?\par
\textbf{Paraphrase:} What is the smallest odd number \textcolor{PromptBoxHeader!55!black}{\textbf{that has four distinct}} prime factors?\par
\textbf{Answer:} $1155$
\end{tcolorbox}

All four perturbed prompts for this example and details of perturbation construction are provided in Appendix~\ref{sec:perturbation-construction}.

\paragraph{Diagnostic setup and drift measurement.}
To characterize semifactual sensitivity before RL training, we sample one response per original prompt from frozen Qwen3-4B-Base on 1,000 questions selected from DAPO-Math-17K.
Sampling details are provided in Appendix~\ref{sec:diagnostic-sampling}.
Let $p_t^{(0)}$ and $p_t^{(k)}$ denote the probabilities assigned to the identical sampled token at position $t$ under the original and the $k$-th perturbed prompts, respectively.
The probability drift measured by the bounded-symmetric distance is defined as follows:
\begin{equation}
    d_t^{(k)}
    =2\tanh\!\left(\left|\log p_t^{(0)}-\log p_t^{(k)}\right|/2\right)
    =\frac{|p_t^{(0)}-p_t^{(k)}|}{\left(p_t^{(0)}+p_t^{(k)}\right)/2}.
\end{equation}

Normalizing by the local mean probability avoids the scale bias of absolute differences and preserves relative probability changes for low-probability tokens.
The distance remains within $[0,2]$, bounding the effect of extreme probability ratios for numerical robustness.
We average across the four perturbations to obtain $d_t=\frac{1}{4}\sum_{k=1}^{4}d_t^{(k)}$.

\paragraph{Heterogeneity in token sensitivity.}
Figure~\ref{fig:motivating-instability}(a) shows both increases and decreases in token probability across a wide range of original confidence levels.
Figure~\ref{fig:motivating-instability}(b) reveals heterogeneity in semifactual sensitivity.
Nonzero probability drift magnitudes span several orders of magnitude, while 12.82\% of positions show no recorded change.
Semifactual sensitivity also varies across token categories, as shown in Figure~\ref{fig:motivating-instability}(c).
Reflection markers adapted from~\citep{wang2025wait} and discourse connectives drawn from~\citep{das2018constructing} exhibit mean drifts of 2.74$\times$ and 2.32$\times$ the overall mean, respectively, whereas mathematical symbols and numbers exhibit 0.47$\times$ and 0.37$\times$, respectively.
This contrast suggests that tokens lexically associated with organizing and reflection are more sensitive to these interventions than mathematical symbols or numbers.
See Appendix~\ref{sec:token-taxonomy} for more details on token category definitions and statistics, and Appendix~\ref{sec:complete-token-cases} for illustrative case studies of token-level semifactual sensitivity.

\begin{wrapfigure}{r}{0.43\textwidth}
\centering
\includegraphics[width=\linewidth]{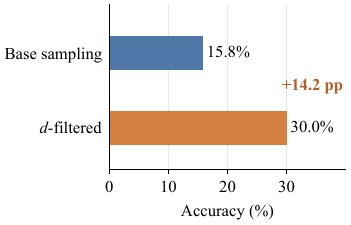}
\caption{\textbf{Semifactual filtering improves rollout accuracy.}
Suppressing high-drift token candidates during rollout generation improves Qwen3-4B-Base's accuracy by 14.2 percentage points on the same 1,000-question diagnostic panel.}
\label{fig:decoding-benefit}
\end{wrapfigure}

\paragraph{Suppressing unstable token candidates improves rollout accuracy.}
To test whether this semifactual sensitivity signal can help improve reasoning accuracy, we evaluate Qwen3-4B-Base under standard sampling and an online $d$-filtered decoding strategy on the same 1,000 selected questions.
Specifically, at each decoding step, we compute drift for every candidate token in the vocabulary, rank non-EOS candidates in descending drift order, and mask the longest prefix whose cumulative probability mass does not exceed 0.8.
We always retain the EOS token, and then renormalize the remaining probabilities.
With model weights unchanged, accuracy rises from 15.8\% to 30.0\% (+14.2 points), as shown in Figure~\ref{fig:decoding-benefit}.
This result demonstrates that semifactual sensitivity provides an actionable token-level signal and motivates its use for finer-grained credit assignment during RLVR training.
More details and additional results are provided in Appendix~\ref{sec:decoding-intervention}.

\WFclear

\section{Semifactual Credit-Augmented Policy Optimization}
\label{sec:scapo}

Building on the decoding benefit of suppressing unstable token candidates, we use semifactual stability to probe potential spurious dependence and refine token-level credit assignment during RLVR training.
Our guiding principle is to reduce the advantages assigned to relatively unstable tokens without granting additional credit for stability alone, since stability does not imply correctness.
In this section, we introduce \methodfull{} (\method{}), a GRPO-based algorithm that combines trajectory-level outcome advantages with token-level semifactual stability for finer-grained credit assignment.
Figure~\ref{fig:semifactual-stability} presents an overview of SCAPO, which consists of semifactual prompt intervention, token probability drift probing, relative stability signal construction, and policy optimization with GRPO advantages augmented by the semifactual credit signal.

\begin{figure}[t]
\centering
\includegraphics[width=0.98\textwidth]{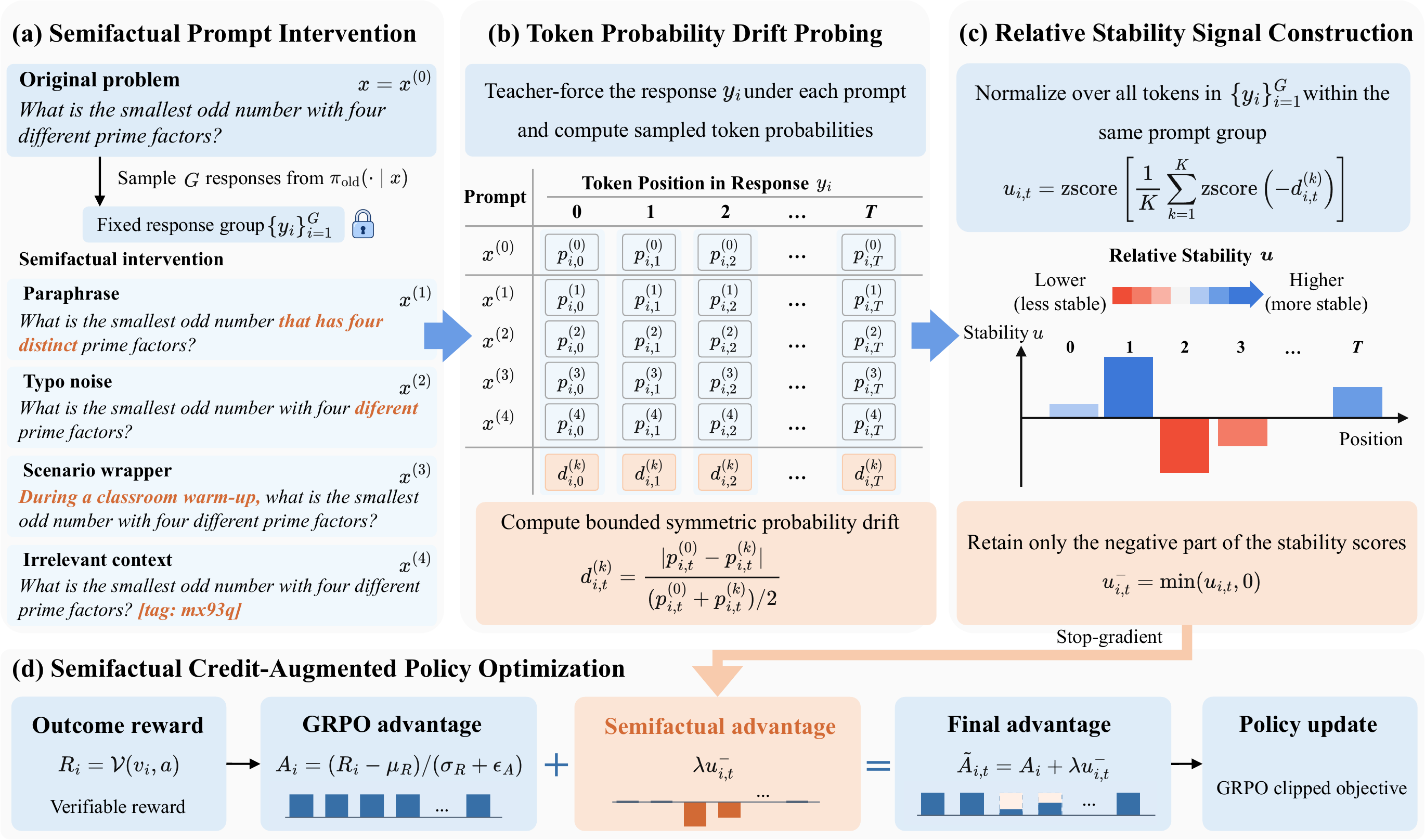}
\caption{Overview of \method{}. (a) Semifactual prompt interventions are paired with a fixed group of sampled responses. (b) Teacher forcing measures sampled-token probability drift under each intervention. (c) Group-wise normalization and aggregation yield relative stability scores, retaining only their negative part. (d) The scaled stability correction, with gradients stopped, is added to the GRPO advantage to produce token-level advantages for policy optimization.}
\label{fig:semifactual-stability}
\end{figure}

\subsection{Preliminaries: Group Relative Policy Optimization}
\label{sec:grpo-preliminaries}

Given a problem--answer pair $(x,a)\sim\mathcal D$, the rollout policy $\pi_{\mathrm{old}}$ samples a group of $G$ responses $\{y_i\}_{i=1}^G$ under the original prompt $x$, where $y_i=(y_{i,1},\ldots,y_{i,T_i})$ denotes the $i$-th response.
Each response receives a verifiable outcome reward $R_i=\mathcal V(y_i,a)\in\{0,1\}$.
GRPO constructs a trajectory-level advantage from the relative rewards within the group~\citep{shao2024deepseekmath}:
\begin{equation}
    A_i=\frac{R_i-\mu_R}{\sigma_R+\epsilon_A},
    \qquad i=1,\ldots,G,
\end{equation}
where $\mu_R$ and $\sigma_R$ are the mean and standard deviation of the group rewards, and $\epsilon_A>0$ is a numerical stabilizer.
The same $A_i$ is applied to every valid token in $y_i$, so it cannot distinguish token-level dependence on task-irrelevant prompt features.

For sampled token $y_{i,t}$, the importance ratio between the current policy $\pi_\theta$ and the rollout policy, together with its clipped counterpart, is
\begin{equation}
\begin{aligned}
    r_{i,t}(\theta)
    &=\frac{\pi_\theta(y_{i,t}\mid x,y_{i,<t})}
    {\pi_{\mathrm{old}}(y_{i,t}\mid x,y_{i,<t})},\\
    \bar r_{i,t}(\theta)
    &=\operatorname{clip}\!\left(
    r_{i,t}(\theta),1-\epsilon_{\mathrm{low}},1+\epsilon_{\mathrm{high}}
    \right).
\end{aligned}
\end{equation}
Following the token-mean formulation~\citep{yu2025dapo}, we update the policy by maximizing the clipped token-level policy gradient objective:
\begin{equation}
\label{eq:grpo-objective}
\mathcal J_{\mathrm{GRPO}}(\theta)
=\mathbb E\!\left[
\frac{1}{\sum_{i=1}^{G}T_i}
\sum_{i=1}^{G}\sum_{t=1}^{T_i}
\min\!\left(r_{i,t}A_i,\bar r_{i,t}A_i\right)
\right].
\end{equation}

\subsection{Fixed-response semifactual probing}
\label{sec:semifactual-stability}

To obtain token-level credit signals, we examine how sampled-token probabilities change under semifactual prompt interventions while holding the response fixed.
For each original prompt $x=x^{(0)}$, we construct $K$ perturbed prompts $x^{(1)},\ldots,x^{(K)}$ using the semifactual perturbations described in Section~\ref{sec:motivating-observation}.
As shown in Figure~\ref{fig:semifactual-stability}(a), the $G$ responses sampled under the original prompt are held fixed across interventions.

We teacher-force the same response under each prompt, keeping both the sampled token $y_{i,t}$ and its response prefix $y_{i,<t}$ unchanged:
\begin{equation}
    p^{(k)}_{i,t}
    =\pi_{\mathrm{old}}\!\left(
    y_{i,t}\mid x^{(k)},y_{i,<t}
    \right),
    \qquad k=0,\ldots,K.
\end{equation}
The rollout policy remains fixed during these probes and is updated as training proceeds.

Using the bounded-symmetric distance from Section~\ref{sec:motivating-observation}, we convert these aligned probabilities into token-level drift, as illustrated in Figure~\ref{fig:semifactual-stability}(b):
\begin{equation}
    d^{(k)}_{i,t}
    =\frac{\left|p^{(0)}_{i,t}-p^{(k)}_{i,t}\right|}
    {\left(p^{(0)}_{i,t}+p^{(k)}_{i,t}\right)/2},
    \qquad k=1,\ldots,K.
\end{equation}
A larger $d^{(k)}_{i,t}$ indicates greater sensitivity of the sampled token's probability to the semifactual prompt intervention.

\subsection{Group-relative stability estimation}
\label{sec:group-relative-stability}

We convert token-level drift into a relative stability signal within each prompt group.
Let $\bm d^{(k)}$ collect the drifts at all valid token positions across the group's $G$ responses for perturbation type $k$.
For any vector $\bm v$ over the valid token positions, define the group-wise standardization
\[
    \operatorname{zscore}_x(\bm v)
    =\frac{\bm v-\mu_x(\bm v)}
    {\sigma_x(\bm v)+\epsilon},
\]
where the mean and standard deviation are computed over all valid response tokens in the group.

As shown in Figure~\ref{fig:semifactual-stability}(c), we standardize negative drift separately for each perturbation type, average the resulting scores, and normalize the aggregate:
\begin{equation}
\label{eq:semifactual-stability}
    \bm u
    =\operatorname{zscore}_x\!\left(
    \frac{1}{K}\sum_{k=1}^{K}
    \operatorname{zscore}_x\!\left(-\bm d^{(k)}\right)
    \right).
\end{equation}
The resulting $u_{i,t}$ measures within-group relative stability, with negative values identifying relatively unstable tokens.
Since stability alone does not imply correctness, we retain only the negative part of the signal:
\begin{equation}
\label{eq:negative-stability}
    u^-_{i,t}=\min(u_{i,t},0).
\end{equation}

\subsection{Policy optimization with semifactual credit}
\label{sec:one-sided-augmentation}

We add the semifactual credit signal to the original GRPO advantage, enabling tokens with the same outcome reward to receive different learning signals.
As shown in Figure~\ref{fig:semifactual-stability}(d), the resulting token-level advantage is
\begin{equation}
\label{eq:modified-advantage}
    \widetilde A_{i,t}
    =A_i+\lambda\,\operatorname{sg}\!\left(u^-_{i,t}\right),
\end{equation}
where $\lambda\ge0$ controls the augmentation strength and $\operatorname{sg}$ denotes stop-gradient.
The semifactual correction is held fixed during each policy update.

This construction ensures $\widetilde A_{i,t}\le A_i$, with equality when $u_{i,t}\ge0$.
When $A_i>0$, the stability correction can weaken the positive reinforcement signal, and when $A_i<0$, it strengthens the negative learning signal.
The method thus refines outcome credit without treating semifactual stability as a token-level correctness label.

Substituting $\widetilde A_{i,t}$ into Eq.~\ref{eq:grpo-objective} gives \method{}'s policy objective:
\begin{equation}
\label{eq:scapo-objective}
\mathcal J_{\text{\method{}}}(\theta)
=\mathbb E\!\left[
\frac{1}{\sum_{i=1}^{G}T_i}
\sum_{i=1}^{G}\sum_{t=1}^{T_i}
\min\!\left(
r_{i,t}\widetilde A_{i,t},
\bar r_{i,t}\widetilde A_{i,t}
\right)
\right].
\end{equation}
\method{} modifies only the advantage, leaving GRPO's importance ratio, clipping, and loss reduction unchanged.
To guide early trajectory selection, we apply semifactual credit augmentation during an initial phase and then continue with standard GRPO.
Specifically, at policy optimization step $n$, the coefficient is
\begin{equation}
\label{eq:lambda-schedule}
\lambda=
\begin{cases}
\lambda_0, & 1\le n\le N_0,\\
0, & n>N_0,
\end{cases}
\end{equation}
where $\lambda_0$ and $N_0$ control the strength and duration of credit augmentation.
Appendix~\ref{sec:optimization-analysis} provides an analysis under a stochastic-gradient model.

\section{Experiments}
\label{sec:experiments}

\subsection{Setup}
\label{sec:experimental-setup}

\paragraph{Datasets.}
We train on DAPO-Math-17K~\citep{yu2025dapo}, a curated dataset of approximately 17,000 competition-level mathematics problems with integer answers.
For \method{}, each training prompt is paired with $K=4$ precomputed perturbed prompts, corresponding to the four perturbation types introduced in Section~\ref{sec:motivating-observation}.

\paragraph{Training.}
We initialize from Qwen3-4B-Base and Qwen3-1.7B-Base~\citep{yang2025qwen3} and train directly with RL for 600 and 1,000 policy optimization steps, respectively.
We use a rollout batch size of 128, an update batch size of 64, and 8 rollouts per prompt.
The maximum response length is 16,384 tokens.
We use the R1-style prompt template, binary rule-based rewards, and a constant learning rate of $10^{-6}$.
For \method{}, we set $\lambda_0=0.01$ and apply credit augmentation for the first $N_0=120$ and $200$ policy optimization steps on the 4B and 1.7B models, respectively.
More detailed training hyperparameters are provided in Appendix~\ref{sec:training-hyperparameters}.

\paragraph{Evaluation.}
We evaluate on in-distribution competition-level mathematical reasoning benchmarks, including AIME 2024--2026~\citep{maaAime}, AMC 2023--2025~\citep{maaAmc}, HMMT 2025--2026~\citep{hmmtArchive}, BRUMO 2025~\citep{brumo2025}, SMT 2025~\citep{smt2025}, Omni-Math~\citep{gao2025omni}, Minerva~\citep{lewkowycz2022solving}, and OlympiadBench~\citep{he2024olympiadbench}.
To assess out-of-distribution generalization, we evaluate NoOp-style distractor variants of AIME~\citep{mirzadeh2025gsm} and ThinkBench perturbations of AIME~\citep{huang2025thinkbench} for robustness to input perturbations.
We further evaluate GPQA-Diamond~\citep{rein2024gpqa}, a graduate-level science question-answering benchmark, to examine transfer beyond mathematics.
All main results use the final checkpoints with identical generation and scoring settings across methods within each benchmark.
We sample at temperature 0.7 and top-$p=0.9$ with a maximum response length of 16,384 tokens, score mathematical responses with Math-Verify~\citep{kydlicek2025mathverify}, and report average accuracy across sampled responses.
Further evaluation details are provided in Appendix~\ref{sec:evaluation-protocol}.

\paragraph{Baselines.}
We consider the following representative RLVR methods:
(1) \textit{GRPO}~\citep{shao2024deepseekmath}: computes group-relative advantages from outcome rewards and applies token-level clipped policy updates.
(2) \textit{GSPO}~\citep{zheng2025gspo}: uses sequence-level importance ratios and clipping for policy optimization.
(3) \textit{SAPO}~\citep{gao2025soft}: replaces hard clipping with smooth, advantage-dependent gating.
(4) \textit{CF-GRPO}~\citep{khandoga2026beyond}: estimates token-level credit through counterfactual masking of reasoning spans.
(5) \textit{FIPO}~\citep{ma2026fipo}: reweights token advantages using discounted future KL.
At each model scale, comparisons match the training data, reward function, prompt template, optimizer, rollout settings, training-step budget, and evaluation protocol.
Method-specific hyperparameters are provided in Appendix~\ref{sec:training-hyperparameters}.

\subsection{Main Results}
\label{sec:experimental-results}

\paragraph{Mathematical reasoning.}
Table~\ref{tab:main-results} shows that \method{} improves over GRPO on all eight mathematical reasoning metrics at both model scales.
AIME accuracy increases by 5.63 and 4.17 percentage points on the 4B and 1.7B models, respectively, while HMMT improves by 4.46 and 2.98 points.
\method{} achieves the highest accuracy on six of eight metrics at 4B and all eight at 1.7B, with the highest average accuracy at both scales.
Additional results show that \method{} achieves the highest Pass@128 on AIME and AMC at both model scales (see Appendix~\ref{sec:passk-results}).

\begin{table}[t]
\caption{Overall accuracy on competition-level mathematical reasoning and out-of-distribution benchmarks (\%). Base denotes the model before RL. AIME variants aggregate 2024–2026, while HMMT and AMC aggregate 2025–2026 and 2023–2025, respectively. Year-specific results are provided in Appendix~\ref{sec:detailed-eval}. Bold and \underline{underline} indicate the best and second-best results, respectively.}
\label{tab:main-results}
\label{tab:robust-results}
\centering
\fontsize{9}{10.3}\selectfont
\renewcommand{\arraystretch}{1.05}
\setlength{\aboverulesep}{1.5pt}
\setlength{\belowrulesep}{1.5pt}
\colorlet{TableGroup}{TableGroup!65!white}
\colorlet{TableMethod}{TableAvg!65!white}
\setlength{\tabcolsep}{2.5pt}
\begin{tabular*}{\textwidth}{@{\extracolsep{\fill}}clrrrrrr>{\columncolor{TableMethod}}r@{}}
\toprule
& \textbf{Benchmark} & \textbf{Base} & \textbf{GRPO} & \textbf{GSPO} & \textbf{SAPO} & \textbf{CF-GRPO} & \textbf{FIPO} & \textbf{\method{}} \\
\midrule
\rowcolor{TableGroup}\multicolumn{9}{c}{\textbf{Qwen3-4B-Base}} \\
\multirow{9}{*}{\rotatebox{90}{In-Distribution}} & AIME 24--26 & 7.64 & 22.08 & 23.47 & 24.03 & 22.64 & \second{25.42} & \best{27.71} \\
& AMC 23--25 & 37.40 & 59.74 & 64.71 & \second{65.45} & 62.06 & 64.57 & \best{65.65} \\
& HMMT 25--26 & 2.08 & 11.71 & 13.69 & 15.18 & 14.09 & \second{15.28} & \best{16.17} \\
& BRUMO 25 & 18.54 & 28.96 & 33.75 & 33.54 & 32.08 & \second{35.00} & \best{37.71} \\
& SMT 25 & 12.74 & 27.12 & 26.06 & 27.12 & 26.77 & \second{28.18} & \best{30.31} \\
& Omni-Math & 26.48 & 35.45 & \second{38.71} & 38.18 & 36.41 & 38.60 & \best{39.42} \\
& Minerva & 31.62 & 40.07 & 39.71 & 41.18 & 41.54 & \best{44.49} & \second{43.38} \\
& Olympiad & 38.58 & 54.30 & 55.64 & 54.75 & 54.75 & \best{57.57} & \second{56.82} \\
& Avg. & 21.88 & 34.93 & 36.97 & 37.43 & 36.29 & \second{38.64} & \best{39.65} \\
\midrule
\multirow{4}{*}{\rotatebox{90}{\fontsize{7.5}{8.5}\selectfont\shortstack{Out-of-\\Distribution}}} & NoOp-AIME 24--26 & 6.74 & 20.56 & 20.14 & 20.42 & \second{21.74} & \second{21.74} & \best{23.61} \\
& ThinkBench-AIME 24--26 & 6.53 & 22.64 & 22.08 & 23.47 & 23.54 & \second{23.89} & \best{26.11} \\
& GPQA-Diamond & 33.78 & 36.26 & \second{41.92} & 41.16 & 36.51 & 40.15 & \best{42.45} \\
& Avg. & 15.68 & 26.48 & 28.05 & 28.35 & 27.26 & \second{28.59} & \best{30.72} \\
\midrule
\rowcolor{TableGroup}\multicolumn{9}{c}{\textbf{Qwen3-1.7B-Base}} \\
\multirow{9}{*}{\rotatebox{90}{In-Distribution}} & AIME 24--26 & 3.61 & 7.71 & 7.36 & \second{10.00} & 6.74 & 7.01 & \best{11.88} \\
& AMC 23--25 & 23.33 & 32.97 & 32.82 & \second{38.58} & 31.99 & 32.48 & \best{41.39} \\
& HMMT 25--26 & 0.69 & 1.98 & 2.08 & \second{4.66} & 2.98 & 1.49 & \best{4.96} \\
& BRUMO 25 & 6.88 & 13.54 & 11.88 & \second{17.71} & 14.37 & 10.00 & \best{19.17} \\
& SMT 25 & 4.83 & 6.96 & 6.84 & \second{10.14} & 8.02 & 8.14 & \best{10.85} \\
& Omni-Math & 16.73 & 20.45 & 21.48 & \second{22.62} & 20.77 & 21.77 & \best{24.74} \\
& Minerva & 21.69 & 27.94 & 27.21 & \second{30.15} & 28.31 & 28.31 & \best{30.88} \\
& Olympiad & 21.81 & 31.31 & 33.68 & \second{37.83} & 33.98 & 32.64 & \best{40.80} \\
& Avg. & 12.45 & 17.86 & 17.92 & \second{21.46} & 18.39 & 17.73 & \best{23.08} \\
\midrule
\multirow{4}{*}{\rotatebox{90}{\fontsize{7.5}{8.5}\selectfont\shortstack{Out-of-\\Distribution}}} & NoOp-AIME 24--26 & 2.36 & 6.11 & 5.97 & \second{8.06} & 6.60 & 6.60 & \best{8.40} \\
& ThinkBench-AIME 24--26 & 3.19 & 7.57 & 7.22 & \second{9.58} & 6.94 & 7.01 & \best{12.01} \\
& GPQA-Diamond & 22.37 & 27.42 & 29.21 & \second{30.96} & 27.61 & 28.24 & \best{31.25} \\
& Avg. & 9.31 & 13.70 & 14.13 & \second{16.20} & 13.72 & 13.95 & \best{17.22} \\
\bottomrule
\end{tabular*}
\end{table}

\paragraph{Out-of-distribution generalization.}
\method{} achieves the highest accuracy on all three out-of-distribution benchmarks at both model scales (Table~\ref{tab:robust-results}).
On GPQA-Diamond, accuracy improves over GRPO from 36.26\% to 42.45\% at 4B and from 27.42\% to 31.25\% at 1.7B, gains of 6.19 and 3.83 percentage points, respectively.
Together with the improvements on NoOp-AIME and ThinkBench-AIME, these results show that the gains extend to both perturbed mathematical problems and scientific reasoning beyond the training domain.

\subsection{Ablation Studies}
\label{sec:ablation-studies}

We ablate the source and sign of token-level credit to examine their roles in \method{}'s performance gains, and the results are shown in Table~\ref{tab:credit-ablation}.

\begin{wraptable}{r}{0.60\textwidth}
\vspace{\dimexpr-\intextsep+6pt\relax}
\centering
\setlength{\abovecaptionskip}{0pt}
\caption{Ablations on Qwen3-1.7B-Base. \method{} uses negative-only corrections.}
\label{tab:credit-ablation}
\fontsize{8}{9.2}\selectfont
\renewcommand{\arraystretch}{1.0}
\setlength{\tabcolsep}{1.5pt}
\setlength{\aboverulesep}{2pt}
\setlength{\belowrulesep}{2pt}
\begin{tabular*}{\linewidth}{@{\extracolsep{\fill}}lrrrrrr@{}}
\toprule
\textbf{Method} & \shortstack{\textbf{AIME}\\\textbf{24--26}} & \shortstack{\textbf{AMC}\\\textbf{23--25}} & \shortstack{\textbf{HMMT}\\\textbf{25--26}} & \shortstack{\textbf{BRUMO}\\\textbf{25}} & \shortstack{\textbf{SMT}\\\textbf{25}} & \textbf{Avg.} \\
\midrule
GRPO & 7.71 & 32.97 & 1.98 & 13.54 & 6.96 & 12.63 \\
\midrule
\textbf{\method{}} & \best{11.88} & \best{41.39} & \best{4.96} & \best{19.17} & \best{10.85} & \best{17.65} \\
\addlinespace[2pt]
\multicolumn{7}{@{}l}{\textit{Credit signal source}} \\
\quad Random shuffle & 8.33 & 32.78 & 2.28 & \second{15.42} & 9.43 & 13.65 \\
\quad Counterfactual & 6.32 & 30.51 & 1.69 & 12.92 & 6.96 & 11.68 \\
\addlinespace[2pt]
\multicolumn{7}{@{}l}{\textit{Correction sign}} \\
\quad All-sign & \second{10.76} & \second{38.83} & \second{3.77} & 15.21 & \second{9.67} & \second{15.65} \\
\quad Positive-only & 9.79 & 34.10 & 2.58 & \second{15.42} & 8.37 & 14.05 \\
\bottomrule
\end{tabular*}
\end{wraptable}

\paragraph{Source of the credit signal.}
The random shuffle ablation permutes semifactual corrections within each response, preserving their values but breaking the original token alignment.
The counterfactual ablation replaces answer-preserving semifactual prompts with answer-changing perturbations, while keeping the credit construction unchanged.
\method{} performs best on the five competition-level benchmarks, followed consistently by random shuffle and the counterfactual control.
On AIME 24--26, \method{} leads these controls by 3.55 and 5.56 percentage points, respectively.
These results support aligning credit with semifactual sensitivity, since drift under answer-changing perturbations may reflect legitimate changes in token predictions rather than spurious dependence.

\paragraph{Sign of the credit correction.}
We compare negative-only, all-sign, and positive-only corrections by retaining negative, all, or positive stability scores.
\method{}'s negative-only correction achieves the highest average accuracy across the five competition-level benchmarks.
This supports reducing credit for relatively unstable tokens without rewarding stability alone.

\subsection{Computational Efficiency}
\label{sec:mechanistic-analysis}

\begin{wrapfigure}{r}{0.43\textwidth}
    \vspace{\dimexpr-\intextsep-18pt\relax}
    \setlength{\abovecaptionskip}{4pt}
    \centering
    \includegraphics[width=\linewidth]{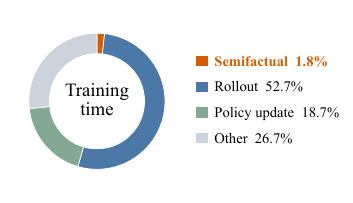}
    \caption{Training time breakdown of \method{} on Qwen3-4B-Base. Other includes validation, reward computation, old policy log-probability computation, and miscellaneous overhead.}
    \label{fig:runtime-breakdown}
\end{wrapfigure}
As shown in Figure~\ref{fig:runtime-breakdown}, semifactual probing and credit construction account for only 1.8\% of the measured training runtime in the 4B experiment, far less than autoregressive rollout.
The probes reuse sampled responses through teacher forcing, avoiding fresh generation under each perturbed prompt.
They are also confined to the initial credit augmentation phase and incur no further cost during subsequent optimization.
Together, these properties keep probing a small fraction of total runtime despite evaluating multiple prompts.
Training efficiency is therefore more strongly affected by major stages such as autoregressive generation, whose cost depends on response length.

\section{Related Work}

\paragraph{Reinforcement learning with verifiable rewards.}
RLVR has advanced mathematical reasoning by optimizing automatically verifiable outcomes, with critic-free group-relative methods enabling direct RL training of base models~\citep{shao2024deepseekmath,guo2025deepseek,zeng2025simplerl}.
Subsequent work improves normalization and sampling~\citep{liu2025understanding,yu2025dapo}, adapts importance weighting and clipping~\citep{zheng2025gspo,gao2025soft}, and promotes exploration through entropy control or off-policy guidance~\citep{cui2025entropy,yan2025learning}.
However, response-level advantages remain too coarse to distinguish individual reasoning tokens.
Studies of reasoning-path coverage further distinguish improved sampling efficiency from expansion of a model's reasoning repertoire~\citep{yue2025does}.
Like prior RLVR methods, our approach optimizes verifiable outcome rewards, but leverages semifactual stability to refine credit assignment at the token level.

\paragraph{Credit assignment for reasoning.}
Fine-grained feedback can be learned from process annotations or outcome labels~\citep{lightman2024let,wang2024mathshepherd,cui2026process}, or estimated through Monte Carlo rollouts~\citep{kazemnejad2025vineppo}.
Token entropy, confidence, eligibility traces, and future policy divergence also guide token selection and weighting~\citep{wang2025beyond,xie2026unlocking,mou2026beyond,ma2026fipo}.
Perturbation- and gradient-based attribution methods estimate how reasoning tokens or spans contribute to final-answer predictions, providing signals for selective supervised fine-tuning and importance-weighted policy updates~\citep{ruan2025enhancing,khandoga2026beyond,li2026outcome}.
Like prior work on fine-grained credit assignment, we use token-level signals to guide policy updates.
Our contribution is to derive negative-only advantage corrections from the sensitivity of a fixed response to semifactual prompt interventions, using teacher-forced token probabilities rather than output-side attribution through span masking or continuation regeneration.

\section{Conclusion}

We presented \method{}, a causally inspired approach for finer-grained token credit assignment in RLVR.
By probing potential token-level spurious dependence through fixed-response sensitivity to semifactual prompt interventions, \method{} applies negative-only corrections to GRPO advantages, selectively reducing credit for relatively unstable tokens.
Across Qwen3-4B-Base and Qwen3-1.7B-Base, \method{} outperforms GRPO on all reported benchmark metrics and achieves the best results on most benchmarks among the compared RLVR methods.
These results highlight that semifactual interventions can serve not only as diagnostic probes of brittle reasoning but also as training signals that complement outcome rewards.
Future work will explore this credit augmentation in larger models, more diverse architectures, and domains beyond mathematical reasoning.

\subsection*{AI Use Statement}
We used generative AI to generate semifactual perturbation data, polish writing, correct grammar, and assist with code debugging and script development.
The authors take responsibility for the final content of this work.

\subsection*{Reproducibility Statement}
Our work is easy to reproduce.
The original datasets and pretrained models used in our experiments are publicly available, and detailed experimental hyperparameters are provided in Appendix~\ref{sec:additional-experimental-details}.
We provide our code, developed on top of the open-source EasyR1 framework, together with all training and evaluation data in the Github repo.

\bibliography{references}
\bibliographystyle{iclr2027_conference}

\appendix
\raggedbottom

\section{Limitations}
\label{sec:limitations}
Our training experiments are limited to mathematical reasoning with dense Qwen3 models of 1.7B and 4B parameters trained on DAPO-Math-17K.
Although GPQA-Diamond provides evidence of transfer to scientific reasoning, the effectiveness of \method{} at larger model and data scales and with broader training domains remains to be evaluated.
Future work should also examine other architectures, including mixture-of-experts (MoE) and hybrid models.

\section{Semifactual Perturbations and Token-Level Analysis}
\label{sec:diagnostic-definitions}

\subsection{Perturbation construction}
\label{sec:perturbation-construction}

We request GPT-5.5~\citep{openai2026gpt55} to generate four perturbed prompts per DAPO-Math-17K problem, using the full system and user prompts in Table~\ref{tab:semifactual-view-definitions}.

\captionof{table}{Full prompts used to generate semifactual perturbations. The user template is filled with the original problem and reference answer.}
\label{tab:semifactual-view-definitions}
\begin{tcolorbox}[
    breakable, colback=PromptBoxBody, colframe=PromptBoxBorder,
    colbacktitle=PromptBoxHeader, coltitle=white,
    title={System prompt}, title after break={System prompt (continued)},
    fonttitle=\footnotesize\bfseries,
    boxrule=0.5pt, arc=1mm, boxsep=0pt,
    left=5pt, right=5pt, top=2pt, bottom=2pt,
    toptitle=1.5pt, bottomtitle=1.5pt,
    before skip=4pt, after skip=4pt
]
\begin{Verbatim}[fontsize={\fontsize{8}{9}\selectfont},breaklines,breakanywhere,breaksymbolleft={},breaksymbolright={}]
Given a math problem and its original answer, produce exactly 4 prompt perturbations.

Core requirements:
- Preserve the same mathematical problem.
- Preserve all numeric values exactly as written.
- Preserve all constraints, operations, requested quantity, and the original answer.
- Produce exactly one perturbation for each perturbation_type: paraphrase, typo_noise, scenario_wrap, irrelevant_context.
- Do not add any condition that helps solve the problem.
- Do not solve the problem.
- Do not reveal reasoning.
- Return JSON only.

Math and LaTeX preservation:
- Treat all math expressions as frozen text. This includes text inside $...$, $$...$$, \(...\), \[...\], and standalone symbolic expressions such as variables, equations, inequalities, fractions, powers, percentages, and units.
- Do not rewrite, reorder, escape, unescape, or modify any LaTeX/math content.
- Preserve every LaTeX command exactly as it appears in the original parsed string: examples include \sin, \angle, \frac, \sqrt, \triangle, \overline.
- In JSON, escape backslashes correctly so that the parsed string contains exactly one backslash before each LaTeX command. Never output doubled LaTeX commands like \\sin, and never output raw control characters such as tab or form feed.
- Do not alter variables, point labels, entity labels, numbers, units, comparison signs, operations, or answer format.

Type-specific requirements:

- paraphrase:
  Rewrite only the natural-language prose that sits outside math expressions. Keep the same scenario, the same entities, the same event and causal order, and every mathematical fact unchanged. Do not introduce a new setting or new objects.

- typo_noise:
  Introduce exactly one character-level edit (insert, delete, substitute, or swap two adjacent characters) in one ordinary English word. The chosen word must NOT be math, LaTeX, a variable, a number, a unit, a proper name, a label, or a mathematical keyword such as "mean", "median", "integer", "prime", "even", "odd". The perturbed text must differ from the original by at least that one non-whitespace character; whitespace-only differences are invalid. The word must still be recognizable to a human reader.

- scenario_wrap:
  Add a light scenario frame around the original problem. This can be a short prefix, a short suffix, or one brief background phrase inserted into the natural-language part of the problem. This should be a small irrelevant context shift, not a full rewrite into a new domain.

  Keep the original people, objects, quantities, units, formulas, variables, constraints, operations, event order, and requested quantity unchanged. Do not replace core entities or items. Do not change "clips" into "packages", "wallet" into "sensor kit", "flowers" into "crates", or a person into a machine.

  Good examples:
    * "At a school craft booth, [original problem]"
    * "During a budgeting exercise, [original problem]"
    * "While tracking a reading plan, [original problem]"
    * "[Original first sentence] This is part of a classroom activity. [remaining original problem]"
    * "[Original problem] This was recorded in a simple daily log."
    * "[Original problem] The setting is an ordinary after-school situation."

  Bad examples:
    * Rewriting the full story into a different domain.
    * Replacing the actors or objects with different actors or objects.
    * Adding hints, extra facts, new conditions, or solution-relevant context.
    * Adding numbers, formulas, variables, quantities, or mathematical facts.

- irrelevant_context:
  Keep the original problem text exactly unchanged. Append exactly one short meaningless distractor at the very end (after the final question mark or full stop).

  Acceptable distractor shapes (pick one; vary the shape and content across samples):
    * A vacuous logical tautology, e.g. "and true is true", "note that 1 = 1", "recall that each thing equals itself".
    * A short random alphanumeric string of length 6 to 12 characters, e.g. "5XeflW1ZJc", "qP7mz39a".
    * A bracketed decorative tag, e.g. "[tag: mx93q]", "[marker-a7f]".

  The distractor MUST NOT contain:
    * Any mathematical fact, quantity, variable, unit, equation, or operator that could interact with the problem.
    * Any hint about the problem, the subject area, or the solution approach.
    * Any statement about the characters, objects, or scenario that appear in the problem.
    * Any framing or meta-language such as "question", "problem", "exercise", "consider", "observe", "for practice", "simple", "word problem".

Return this JSON shape:
{
  "perturbations": [
    {
      "perturbed_question": "...",
      "perturbation_type": "paraphrase | typo_noise | scenario_wrap | irrelevant_context"
    }
  ]
}
\end{Verbatim}
\end{tcolorbox}
\begin{tcolorbox}[
    colback=PromptBoxBody, colframe=PromptBoxBorder,
    colbacktitle=PromptBoxHeader, coltitle=white,
    title={User prompt}, fonttitle=\footnotesize\bfseries,
    boxrule=0.5pt, arc=1mm, boxsep=0pt,
    left=5pt, right=5pt, top=2pt, bottom=2pt,
    toptitle=1.5pt, bottomtitle=1.5pt,
    before skip=4pt, after skip=4pt
]
\begin{Verbatim}[fontsize={\fontsize{8}{9}\selectfont},breaklines]
Original problem:
{problem}

Original answer:
{answer}
\end{Verbatim}
\end{tcolorbox}

\begin{tcolorbox}[
    colback=PromptBoxBody, colframe=PromptBoxBorder,
    colbacktitle=PromptBoxHeader, coltitle=white,
    title={Example: Four semifactual perturbations},
    fonttitle=\small\bfseries, fontupper=\small,
    boxrule=0.5pt, arc=1.5mm, boxsep=0pt,
    left=6pt, right=6pt, top=4pt, bottom=4pt,
    toptitle=3pt, bottomtitle=3pt,
    before skip=6pt, after skip=6pt
]
\textbf{Original:} What is the smallest odd number with four different prime factors?\par\smallskip
\textbf{Paraphrase:} What is the smallest odd number \textcolor{PromptBoxHeader!55!black}{\textbf{that has four distinct}} prime factors?\par\smallskip
\textbf{Typo noise:} What is the smallest odd number with four \textcolor{PromptBoxHeader!55!black}{\textbf{diferent}} prime factors?\par\smallskip
\textbf{Scenario wrapper:} \textcolor{PromptBoxHeader!55!black}{\textbf{During a classroom warm-up,}} what is the smallest odd number with four different prime factors?\par\smallskip
\textbf{Irrelevant context:} What is the smallest odd number with four different prime factors? \textcolor{PromptBoxHeader!55!black}{\textbf{[tag: mx93q]}}\par\smallskip
\textbf{Answer:} $1155$
\end{tcolorbox}

\subsection{Diagnostic setup}
\label{sec:diagnostic-sampling}

We randomly select 1,000 DAPO-Math-17K questions and their four perturbed prompts to form a fixed diagnostic subset dataset.

Frozen Qwen3-4B-Base generates one response per original prompt, with temperature 1.0, top-$p=1.0$, and a limit of 8,192 response tokens.
We then teacher-force the same response tokens under the original prompt and all four perturbed prompts.
Each response token contributes one mean drift $d_t$ across the four perturbations, yielding 870,586 original response token drifts.

\subsection{Token taxonomy and statistics}
\label{sec:token-taxonomy}

We assign each decoded response token to one of the eight mutually exclusive categories in Table~\ref{tab:token-category-statistics}.
Reflection markers are identified using a 16-entry lexicon adapted from prior work~\citep{wang2025wait} and our generation traces.
Discourse connectives are identified using 80 single-part DiMLex-Eng forms that contain no whitespace~\citep{das2018constructing}.
Lexical matching ignores case, surrounding whitespace, and edge punctuation, with reflection markers taking precedence over discourse connectives.

\begin{table}[H]
\centering
\begin{threeparttable}
\caption{Token categories and drift statistics underlying Figure~\ref{fig:motivating-instability}(c).
\textit{Share} is the percentage of the 870,586 response token positions assigned to each category. \textit{Mean} is the category's average drift $d_t$. \textit{Relative mean} is this average divided by the overall mean (0.01640).
\tokspace{} denotes a space and \texttt{\textbackslash n} a newline.}
\label{tab:token-category-statistics}
\footnotesize
\setlength{\tabcolsep}{2.2pt}
\renewcommand{\arraystretch}{1.15}
\begin{tabular*}{\linewidth}{@{\extracolsep{\fill}}>{\raggedright\arraybackslash}p{0.155\linewidth}>{\raggedright\arraybackslash}p{0.275\linewidth}>{\raggedright\arraybackslash}p{0.195\linewidth}rrr@{}}
\toprule
\textbf{Category} & \textbf{Classification rule} & \textbf{Token examples} & \textbf{Share (\%)} & \textbf{Mean} $d_t$ & \textbf{Relative mean} \\
\midrule
Reflection marker\tnote{a} & Matches the reflection lexicon after lexical normalization. & \tokspace\texttt{any}, \tokspace\texttt{check}, \texttt{Now}, \tokspace\texttt{However}, \tokspace\texttt{verify}, \tokspace\texttt{again} & 0.29 & 0.04496 & 2.74 \\
\addlinespace[2pt]
Discourse connective & Matches the 80-form DiMLex-Eng lexicon after normalization. Excludes reflection markers. & \tokspace\texttt{and}, \tokspace\texttt{for}, \tokspace\texttt{if}, \tokspace\texttt{or}, \texttt{Thus}, \tokspace\texttt{Therefore} & 2.68 & 0.03807 & 2.32 \\
\addlinespace[2pt]
Word & Contains a Unicode letter and matches neither lexical category nor the math-symbol rule. & \tokspace\texttt{the}, \tokspace\texttt{of}, \tokspace\texttt{is}, \tokspace\texttt{to}, \tokspace\texttt{we}, \tokspace\texttt{in} & 40.92 & 0.02506 & 1.53 \\
\addlinespace[2pt]
Punctuation & All non-whitespace characters have Unicode category \texttt{P*}. Excludes math symbols. & \texttt{,}\quad\texttt{.}\quad\texttt{:}\quad\texttt{\#}\quad\texttt{),}\quad\texttt{:\textbackslash n} & 6.75 & 0.01956 & 1.19 \\
\addlinespace[2pt]
Math symbol & Begins with \texttt{\textbackslash} or contains only characters from the fixed math/\LaTeX{} symbol set. Excludes numbers. & \tokspace\texttt{\textbackslash}, \tokspace\texttt{=}, \tokspace\texttt{+}, \tokspace\texttt{-}, \texttt{\{}, \texttt{\}} & 26.62 & 0.00768 & 0.47 \\
\addlinespace[2pt]
Number & Matches a signed integer, decimal, or percentage expression. & \texttt{0}, \texttt{1}, \texttt{2}, \texttt{3}, \texttt{4}, \texttt{5} & 15.31 & 0.00612 & 0.37 \\
\addlinespace[2pt]
\mbox{Whitespace-only} & Decodes entirely to Unicode whitespace. & \tokspace{}, \tokspace\tokspace{}, \texttt{\textbackslash n}, \texttt{\textbackslash n\textbackslash n} & 6.71 & 0.00718 & 0.44 \\
\addlinespace[2pt]
Other & Control tokens, undecodable UTF-8 fragments, or remaining forms not covered above. & \texttt{<|endoftext|>}, \texttt{\textasciigrave\textasciigrave\textasciigrave}, \texttt{\textasciigrave\textasciigrave}, \texttt{\$,} & 0.72 & 0.02925 & 1.78 \\
\bottomrule
\end{tabular*}
\begin{tablenotes}[flushleft]
\footnotesize
\item[a] The reflection lexicon contains the following 16 words: \texttt{again}, \texttt{ah}, \texttt{alternative}, \texttt{alternatively}, \texttt{another}, \texttt{any}, \texttt{but}, \texttt{check}, \texttt{hmm}, \texttt{however}, \texttt{maybe}, \texttt{now}, \texttt{oh}, \texttt{other}, \texttt{verify}, \texttt{wait}.
\end{tablenotes}
\end{threeparttable}
\end{table}

\begin{samepage}
\subsection{Complete-response case studies}
\label{sec:complete-token-cases}

Figures~\ref{fig:complete-case-interest} and~\ref{fig:complete-case-change} illustrate two cases with their original and perturbed prompts.
Response tokens are shaded by their mean probability drift across the four prompt perturbations, with darker shading indicating greater sensitivity.
These examples show that even when the final answer is correct, the probabilities of individual response tokens can change substantially under answer-preserving prompt perturbations.

\end{samepage}

\begin{figure}[H]
\centering
\includegraphics[width=\textwidth,height=0.90\textheight,keepaspectratio]{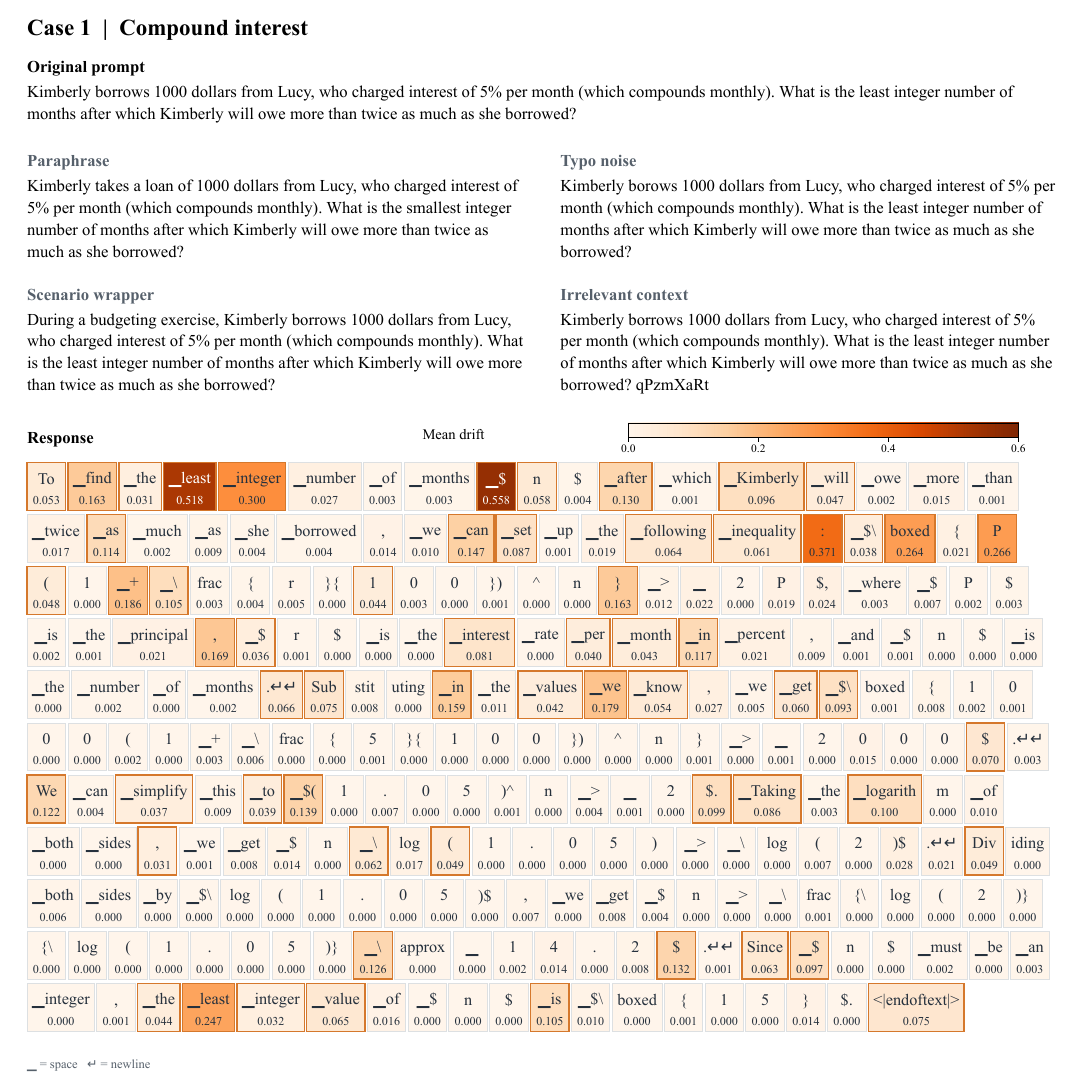}
\caption{Token-level semifactual sensitivity in a compound-interest problem.}
\label{fig:complete-case-interest}
\end{figure}

\begin{figure}[H]
\centering
\includegraphics[width=\textwidth]{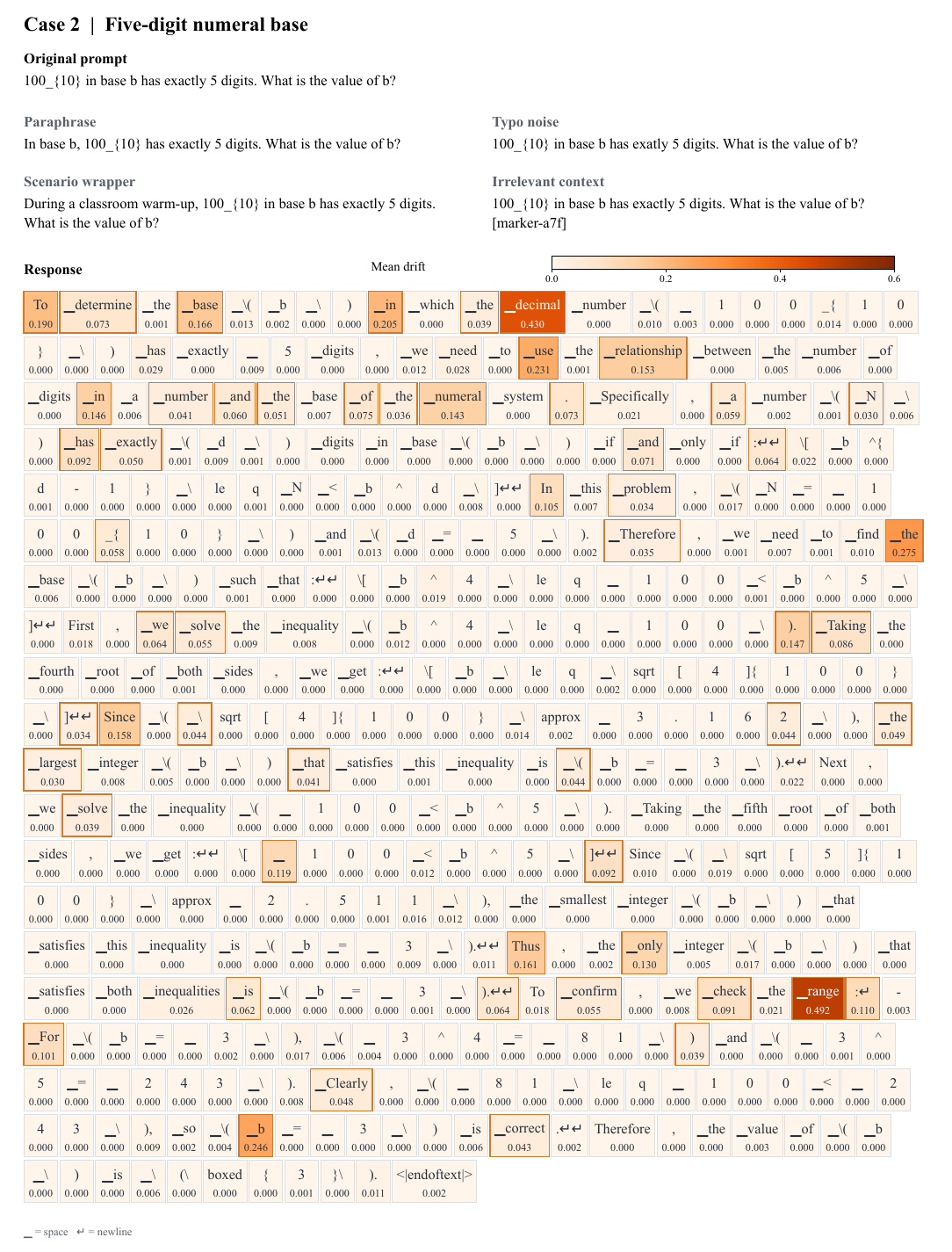}
\caption{Token-level semifactual sensitivity in a numeral-base problem.}
\label{fig:complete-case-change}
\end{figure}

\clearpage
\section{\texorpdfstring{$d$}{d}-Filtered Decoding}
\label{sec:decoding-intervention}

\subsection{Decoding Setup}
\label{sec:decoding-setup}
We use frozen Qwen3-4B-Base on the 1,000-question panel in Appendix~\ref{sec:diagnostic-sampling}.
Both base sampling and $d$-filtered decoding use the same prompt template, temperature of 1.0, top-$p=1.0$, and a maximum response length of 8,192 tokens.
We cap the cumulative clean-policy probability mass of masked tokens at 0.8, which is the sole filtering hyperparameter.
For each question, we independently generate one response under each decoding strategy.

\subsection{Filtering Procedure}
\label{sec:decoding-filtering}
At each decoding step, we compute the mean bounded-symmetric drift of every vocabulary token across the four perturbed prompts.
We sort all token candidates by decreasing drift.
We mask token candidates in this order until adding the next candidate would exceed the mass cap.
The EOS token is always retained.
The cap limits the total probability mass removed, not the fraction of tokens masked.

Writing $p(v)$ for the clean-policy probability at the current prefix and $\mathcal M$ for the masked set, the filtered sampling distribution is
\begin{equation}
    p_{\mathrm{filter}}(v)
    = \frac{p(v)\,\mathbf{1}\{v\notin\mathcal M\}}
           {1-\sum_{w\in\mathcal M}p(w)}.
    \label{eq:diagnostic-filter}
\end{equation}

We sample a proposal from the clean policy and accept it if it is unmasked.
Otherwise, we reject it and resample from the distribution renormalized over unmasked tokens.
Drift and the masked set are recomputed at each step using the current filtered-response prefix.

\subsection{Results}
\label{sec:decoding-results}
As shown in Table~\ref{tab:before-rl-decoding}, filtering yields 202 incorrect-to-correct and 60 correct-to-incorrect transitions, a gain of 14.20 percentage points.
On average, 50.58 token proposals are rejected and resampled per filtered response (4.58\% of decoding steps), with at least one such rejection in each of the 1,000 filtered responses.

\begin{table}[H]
\centering
\caption{Decoding results on 1,000 questions with frozen Qwen3-4B-Base. The lower block reports question counts for the four combinations of baseline and filtered answer correctness.}
\label{tab:before-rl-decoding}
\small
\begin{tabular}{@{}lrr@{}}
\toprule
Metric & Base sampling & $d$-filtered \\
\midrule
Correct answers & 158 & 300 \\
Accuracy (\%) & 15.80 & \textbf{30.00} \\
\midrule
\multicolumn{3}{@{}l}{\textit{Paired correctness outcomes (question counts)}} \\
& Filtered correct & Filtered incorrect \\
Baseline correct & 98 & 60 \\
Baseline incorrect & 202 & 640 \\
\bottomrule
\end{tabular}
\end{table}

\section{Optimization Analysis}
\label{sec:optimization-analysis}

\newtheorem{scapolemma}{Lemma}[section]
\newtheorem{scapotheorem}[scapolemma]{Theorem}

We analyze credit augmentation under a stochastic-gradient model with a fixed objective $F$.

\begin{scapolemma}[Bounded auxiliary gradient]
\label{lem:bounded-augmentation}
Let $\widehat g_\lambda$ and $\widehat g_0$ be the batch gradients of Eqs.~\ref{eq:scapo-objective} and~\ref{eq:grpo-objective} on the same complete prompt groups, with advantages and stability scores held fixed.
If importance ratios are at most $R$ and token log-probability gradients have norm at most $B$, then
\begin{equation}
\label{eq:auxiliary-gradient-bound}
\|\widehat g_\lambda-\widehat g_0\|\le\lambda H,
\qquad H:=RB/2.
\end{equation}
\end{scapolemma}
\begin{proof}
Within each group, write $\langle\cdot\rangle$ for the valid-token mean and $c=(-u)_+$.
Eq.~\ref{eq:semifactual-stability} gives $\langle u\rangle=0$ and $\langle u^2\rangle\le1$, hence
\[
\langle c\rangle=\tfrac12\langle|u|\rangle
\le\tfrac12\sqrt{\langle u^2\rangle}\le\tfrac12.
\]
At differentiable points, the coefficient of a token's log-probability gradient in the clipped surrogate is $rA$, $r\max(A,0)$, or $r\min(A,0)$, each $r$-Lipschitz in $A$.
Replacing $A$ by $A-\lambda c$ therefore changes the group gradient by at most $\lambda RB\langle c\rangle\le\lambda H$.
Token-mean aggregation across groups preserves this bound.
\end{proof}

\paragraph{Stochastic-gradient model.}
Consider
\begin{equation}
\label{eq:analysis-update}
\theta_{n+1}=\theta_n+\eta\widehat v_n,
\qquad
\widehat v_n=\widehat g_n+\lambda_n\widehat h_n,
\end{equation}
where $\widehat g_n$ is the baseline gradient and $\lambda_n\widehat h_n$ is the gradient difference induced by credit augmentation, with $\|\widehat h_n\|\le H$ by Lemma~\ref{lem:bounded-augmentation}.
Let $\mathbb E_n$ denote the expectation conditional on the history before update $n$.
Assume $F$ is $L$-smooth, bounded above by $F_{\sup}$, and
\begin{equation}
\label{eq:analysis-assumptions}
\|\mathbb E_n\widehat g_n-\nabla F(\theta_n)\|\le\beta,
\qquad
\mathbb E_n\|\widehat v_n-\mathbb E_n\widehat v_n\|^2\le\sigma^2.
\end{equation}
The two gradient components may be correlated.
The parameter $\beta$ allows bias in the baseline direction.

\begin{scapotheorem}[Finite-duration augmentation]
\label{thm:finite-duration-augmentation}
Under the above assumptions, let $0<\eta\le1/L$ and use the schedule in Eq.~\ref{eq:lambda-schedule}.
For any $T\ge1$, set $\Delta_F=F_{\sup}-\mathbb EF(\theta_1)$ and $m_T=\min(T,N_0)$.
Then
\begin{equation}
\label{eq:finite-duration-bound}
\frac1T\sum_{n=1}^T\mathbb E\|\nabla F(\theta_n)\|^2
\le
\frac{2\Delta_F}{\eta T}+L\eta\sigma^2+\beta^2
+\bigl(2\beta H\lambda_0+H^2\lambda_0^2\bigr)\frac{m_T}{T}.
\end{equation}
\end{scapotheorem}

\begin{proof}
Write $v_n=\nabla F(\theta_n)$ and $b_n=\mathbb E_n\widehat v_n-v_n$, so $\|b_n\|\le\beta+H\lambda_n$.
Smoothness and Eq.~\ref{eq:analysis-assumptions} give
\[
\begin{aligned}
\mathbb E_nF(\theta_{n+1})
&\ge F(\theta_n)+\eta\langle v_n,v_n+b_n\rangle
-\tfrac{L\eta^2}{2}\bigl(\|v_n+b_n\|^2+\sigma^2\bigr)\\
&\ge F(\theta_n)+\tfrac{\eta}{2}\|v_n\|^2
-\tfrac{\eta}{2}(\beta+H\lambda_n)^2-\tfrac{L\eta^2}{2}\sigma^2,
\end{aligned}
\]
where the second line uses $\eta L\le1$ and
$2\langle v,v+b\rangle=\|v\|^2+\|v+b\|^2-\|b\|^2$.
Taking expectations, summing over $n$, and using $\mathbb EF(\theta_{T+1})\le F_{\sup}$ yields the claim.
\end{proof}

\paragraph{Implication for the schedule.}
For an unbiased baseline ($\beta=0$), the augmentation term is $H^2\lambda_0^2\min(T,N_0)/T$.
With fixed $\lambda_0$ and $N_0$, this contribution to the average stationarity bound decays as $N_0/T$ after augmentation ends.
Here $T$ is an arbitrary observation horizon, not a prescribed training budget.

\section{Additional Experimental Details}
\label{sec:additional-experimental-details}

\subsection{Training hyperparameters}
\label{sec:training-hyperparameters}

Table~\ref{tab:training-config} lists shared training hyperparameters and method-specific settings.
A rule-based verifier assigns reward 1 to a correct boxed answer and 0 otherwise, and no separate format reward is used.
All models share the following prompt template for mathematical reasoning during training and inference:

\begin{tcolorbox}[
    colback=PromptBoxBody, colframe=PromptBoxBorder,
    colbacktitle=PromptBoxHeader, coltitle=white,
    title={Prompt template}, fonttitle=\small\bfseries,
    fontupper=\small, boxrule=0.5pt, arc=1mm, boxsep=0pt,
    left=6pt, right=6pt, top=4pt, bottom=4pt,
    toptitle=2pt, bottomtitle=2pt, before skip=6pt, after skip=6pt
]
\texttt{\{question\}}\par\medskip
Please reason step by step, and put your final answer within \verb|\boxed{}|.
\end{tcolorbox}

All methods, including GSPO and SAPO, use token-mean loss reduction for better performance on mathematical reasoning tasks~\citep{yu2025dapo,wang2026arlarena}.

\begin{table}[H]
\caption{Shared and method-specific training hyperparameters. Batch sizes count prompts unless stated otherwise. Each rollout step contains two policy optimization steps. Clipping offsets are relative to 1.}
\label{tab:training-config}
\label{tab:method-hyperparameters}
\begin{center}
\small
\setlength{\tabcolsep}{5pt}
\renewcommand{\arraystretch}{1.06}
\begin{tabular*}{\textwidth}{@{\extracolsep{\fill}}lcc@{}}
\toprule
\textbf{Hyperparameter} & \textbf{Qwen3-4B-Base} & \textbf{Qwen3-1.7B-Base} \\
\midrule
\multicolumn{3}{@{}l}{\textit{Data and rollout settings}} \\
Training dataset & \multicolumn{2}{c}{DAPO-Math-17K~\citep{yu2025dapo}} \\
Maximum prompt length & \multicolumn{2}{c}{1,024} \\
Maximum response length & \multicolumn{2}{c}{16,384} \\
Rollout batch size (prompts) & \multicolumn{2}{c}{128} \\
Responses per prompt ($G$) & \multicolumn{2}{c}{8} \\
Sampling temperature / top-$p$ & \multicolumn{2}{c}{$1.0\;/\;1.0$} \\
\midrule
\multicolumn{3}{@{}l}{\textit{Optimization settings}} \\
Optimizer & \multicolumn{2}{c}{AdamW, $\beta_1=0.9$, $\beta_2=0.999$} \\
Learning rate & \multicolumn{2}{c}{$10^{-6}$} \\
Learning-rate schedule & \multicolumn{2}{c}{Constant} \\
Weight decay & \multicolumn{2}{c}{$0.01$} \\
Gradient clipping norm & \multicolumn{2}{c}{$1.0$} \\
Update mini-batch size (prompts) & \multicolumn{2}{c}{64} \\
Loss reduction & \multicolumn{2}{c}{Token-mean} \\
KL regularization & \multicolumn{2}{c}{Disabled} \\
\midrule
\multicolumn{3}{@{}l}{\textit{Training settings}} \\
Rollout steps & 300 & 500 \\
Policy optimization steps & 600 & 1,000 \\
Mixed precision & \multicolumn{2}{c}{BF16} \\
\midrule
\multicolumn{3}{@{}l}{\textit{Method-specific settings}} \\
\method{}: Correction strength ($\lambda_0$) & \multicolumn{2}{c}{$0.01$} \\
\method{}: Augmentation duration ($N_0$) & 120 & 200 \\
\multicolumn{3}{@{}l}{GRPO / CF-GRPO / FIPO / \method{}:} \\
\quad Policy clipping (lower / upper) & \multicolumn{2}{c}{$0.20\;/\;0.28$} \\
\quad Dual-clip coefficient & \multicolumn{2}{c}{$10.0$} \\
GSPO: Policy clipping (lower / upper) & \multicolumn{2}{c}{$3\times10^{-4}\;/\;4\times10^{-4}$} \\
SAPO: Gate temperatures ($\tau_{\mathrm{pos}}/\tau_{\mathrm{neg}}$) & \multicolumn{2}{c}{$1.0\;/\;1.05$} \\
CF-GRPO: Maximum probed spans per response & \multicolumn{2}{c}{10} \\
CF-GRPO: Span length & \multicolumn{2}{c}{5--50} \\
CF-GRPO: Token-weight bounds & \multicolumn{2}{c}{$[0.5,4.0]$} \\
FIPO: Future-KL decay & \multicolumn{2}{c}{$32.0$} \\
FIPO: Influence-weight clipping & \multicolumn{2}{c}{$[1.0,1.2]$} \\
FIPO: Safety threshold & \multicolumn{2}{c}{$10.0$} \\
\bottomrule
\end{tabular*}
\end{center}
\end{table}

Each rollout batch contains 128 prompt groups and is optimized once in mini-batches of 64 groups (512 responses), giving two policy optimization steps per rollout step.
Thus, $N_0=120/200$ policy optimization steps correspond to 60/100 rollout steps for 4B/1.7B.
Diagnostic checkpoint indices use rollout steps unless stated otherwise.

\subsection{Evaluation protocol and training curves}
\label{sec:evaluation-protocol}

Table~\ref{tab:evaluation-config} summarizes the evaluation datasets and sample counts.

\begin{table}[H]
\caption{Evaluation benchmarks and sample counts. Slash-separated problem counts follow the listed years. $^{*}$For GPQA-Diamond, we enumerate all $4!=24$ permutations of the multiple-choice options to avoid contamination.}
\label{tab:evaluation-config}
\begin{center}
\small
\setlength{\tabcolsep}{5pt}
\renewcommand{\arraystretch}{1.08}
\begin{tabular*}{\textwidth}{@{\extracolsep{\fill}}lcc@{}}
\toprule
\textbf{Benchmark} & \textbf{Problems} & \textbf{Samples per problem} \\
\midrule
AIME 2024 / 2025 / 2026 & 30 / 30 / 30 & 16 \\
AMC 2023 / 2024 / 2025 & 40 / 45 / 42 & 16 \\
HMMT Feb.\ 2025 / 2026 & 30 / 33 & 16 \\
BRUMO 2025 & 30 & 16 \\
SMT 2025 & 53 & 16 \\
Omni-Math & 2,821 & 1 \\
Minerva & 272 & 1 \\
OlympiadBench & 674 & 1 \\
\midrule
NoOp-AIME 2024 / 2025 / 2026 & 30 / 30 / 30 & 16 \\
ThinkBench-AIME 2024 / 2025 / 2026 & 120 / 120 / 120 & 4 \\
GPQA-Diamond & 198 & 24$^{*}$ \\
\bottomrule
\end{tabular*}
\end{center}
\end{table}

For Omni-Math, we use the 2,821-problem subset designed for rule-based evaluation.
We evaluate the final checkpoints using identical prompts and sample counts across methods, with temperature 0.7, top-$p=0.9$, and a maximum response length of 16,384 tokens.
Responses on the mathematical benchmarks are scored with Math-Verify~\citep{kydlicek2025mathverify}.

We perform rule-based validation before training and every 40 policy optimization steps.
Figure~\ref{fig:grpo-semifact-aime} presents the corresponding accuracy curves.

\section{Additional Experimental Results}
\label{sec:additional-results}

\subsection{Detailed Benchmark Results}
\label{sec:detailed-eval}

Table~\ref{tab:appendix-standard-detail} shows the yearly results underlying Table~\ref{tab:main-results}.

\begin{table}[H]
\caption{Yearly accuracy (\%) for Table~\ref{tab:main-results}. Bold and underlining denote the best and second-best results.}
\label{tab:appendix-standard-detail}
\label{tab:appendix-robust-detail}
\begin{center}
\fontsize{8}{9.5}\selectfont
\setlength{\tabcolsep}{1.5pt}
\renewcommand{\arraystretch}{1.05}
\setlength{\aboverulesep}{2pt}
\setlength{\belowrulesep}{2pt}
\colorlet{TableGroup}{TableGroup!65!white}
\begin{tabular*}{\textwidth}{@{\extracolsep{\fill}}l*{14}{r}@{}}
\toprule
& \multicolumn{8}{c}{\textbf{In-Distribution}} & \multicolumn{6}{c}{\textbf{Out-of-Distribution}} \\
\cmidrule(lr){2-9}\cmidrule(lr){10-15}
\textbf{Method} & \multicolumn{3}{c}{\textbf{AIME}} & \multicolumn{3}{c}{\textbf{AMC}} & \multicolumn{2}{c}{\textbf{HMMT}} & \multicolumn{3}{c}{\textbf{NoOp-AIME}} & \multicolumn{3}{c}{\textbf{ThinkBench-AIME}} \\
\cmidrule(lr){2-4}\cmidrule(lr){5-7}\cmidrule(lr){8-9}\cmidrule(lr){10-12}\cmidrule(lr){13-15}
& \textbf{24} & \textbf{25} & \textbf{26} & \textbf{23} & \textbf{24} & \textbf{25} & \textbf{25} & \textbf{26} & \textbf{24} & \textbf{25} & \textbf{26} & \textbf{24} & \textbf{25} & \textbf{26} \\
\midrule
\rowcolor{TableGroup}\multicolumn{15}{c}{\textbf{Qwen3-4B-Base}} \\
Base & 8.33 & 7.92 & 6.67 & 47.66 & 32.50 & 32.89 & 1.04 & 3.03 & 8.33 & 7.29 & 4.58 & 8.96 & 3.75 & 6.88 \\
GRPO & 23.54 & 23.13 & 19.58 & 66.41 & 55.69 & 57.74 & 10.83 & 12.50 & 21.25 & 21.25 & 19.17 & 25.42 & 22.92 & 19.58 \\
GSPO & 27.08 & 22.92 & 20.42 & \best{73.59} & 61.39 & 59.82 & 10.42 & 16.67 & 19.58 & \second{21.46} & 19.38 & 24.58 & 22.71 & 18.96 \\
SAPO & 26.46 & 24.17 & 21.46 & \second{72.50} & 59.31 & \best{65.33} & 12.29 & \second{17.80} & 21.25 & 21.04 & 18.96 & \best{27.71} & 22.50 & 20.21 \\
CF-GRPO & 24.17 & 24.17 & 19.58 & 68.44 & 59.03 & 59.23 & \second{13.13} & 14.96 & \best{23.75} & \second{21.46} & 20.00 & 26.88 & \second{23.33} & 20.42 \\
FIPO & \second{27.50} & \second{27.08} & \second{21.67} & 69.06 & \second{61.81} & 63.24 & 12.08 & \best{18.18} & \second{23.33} & 20.63 & \second{21.25} & 26.67 & 22.29 & \second{22.71} \\
\midrule
\textbf{\method{}} & \best{29.17} & \best{27.71} & \best{26.25} & 70.31 & \best{63.19} & \second{63.84} & \best{13.96} & \best{18.18} & 22.71 & \best{25.62} & \best{22.50} & \second{27.50} & \best{26.46} & \best{24.38} \\
\midrule
\rowcolor{TableGroup}\multicolumn{15}{c}{\textbf{Qwen3-1.7B-Base}} \\
Base & 3.96 & 3.75 & 3.13 & 32.66 & 17.50 & 20.68 & 0.00 & 1.33 & 2.71 & 1.67 & 2.71 & 4.58 & 3.54 & 1.46 \\
GRPO & 10.63 & 6.67 & \second{5.83} & \second{46.41} & 28.19 & 25.30 & 1.25 & 2.65 & \second{8.54} & 5.00 & 4.79 & 9.79 & 6.46 & \second{6.46} \\
GSPO & 9.58 & 7.08 & 5.42 & 44.38 & 27.08 & 27.98 & 1.88 & 2.27 & 6.46 & 6.04 & 5.42 & 10.00 & 6.46 & 5.21 \\
SAPO & \second{13.54} & \second{8.54} & \best{7.92} & \best{51.09} & \second{33.61} & \second{31.99} & \second{3.54} & \best{5.68} & 7.50 & \second{7.71} & \best{8.96} & \second{11.67} & \second{8.75} & \best{8.33} \\
CF-GRPO & 9.17 & 5.21 & \second{5.83} & \second{46.41} & 23.75 & 27.08 & 2.50 & 3.41 & \best{8.75} & 6.04 & 5.00 & 9.58 & 6.88 & 4.38 \\
FIPO & 8.75 & 7.50 & 4.79 & 46.25 & 25.28 & 27.08 & 1.04 & 1.89 & 7.92 & 5.42 & 6.46 & 10.21 & 5.00 & 5.83 \\
\midrule
\textbf{\method{}} & \best{14.79} & \best{12.92} & \best{7.92} & \best{51.09} & \best{37.08} & \best{36.76} & \best{4.79} & \second{5.11} & 8.33 & \best{9.17} & \second{7.71} & \best{16.25} & \best{11.46} & \best{8.33} \\
\bottomrule
\end{tabular*}
\end{center}
\end{table}

\clearpage
\subsection{Pass@\texorpdfstring{$k$}{k} Results}
\label{sec:passk-results}

We report Pass@$k$ on AIME and AMC, measuring the probability of obtaining at least one correct answer within $k$ attempts.
We estimate the curves from 128 responses per problem using the standard unbiased estimator, with the same sampling and scoring settings as the main evaluation.
Figure~\ref{fig:appendix-passk} shows that \method{} achieves the highest Pass@128 at both model scales.

\begin{figure}[H]
\centering
\includegraphics[width=\textwidth]{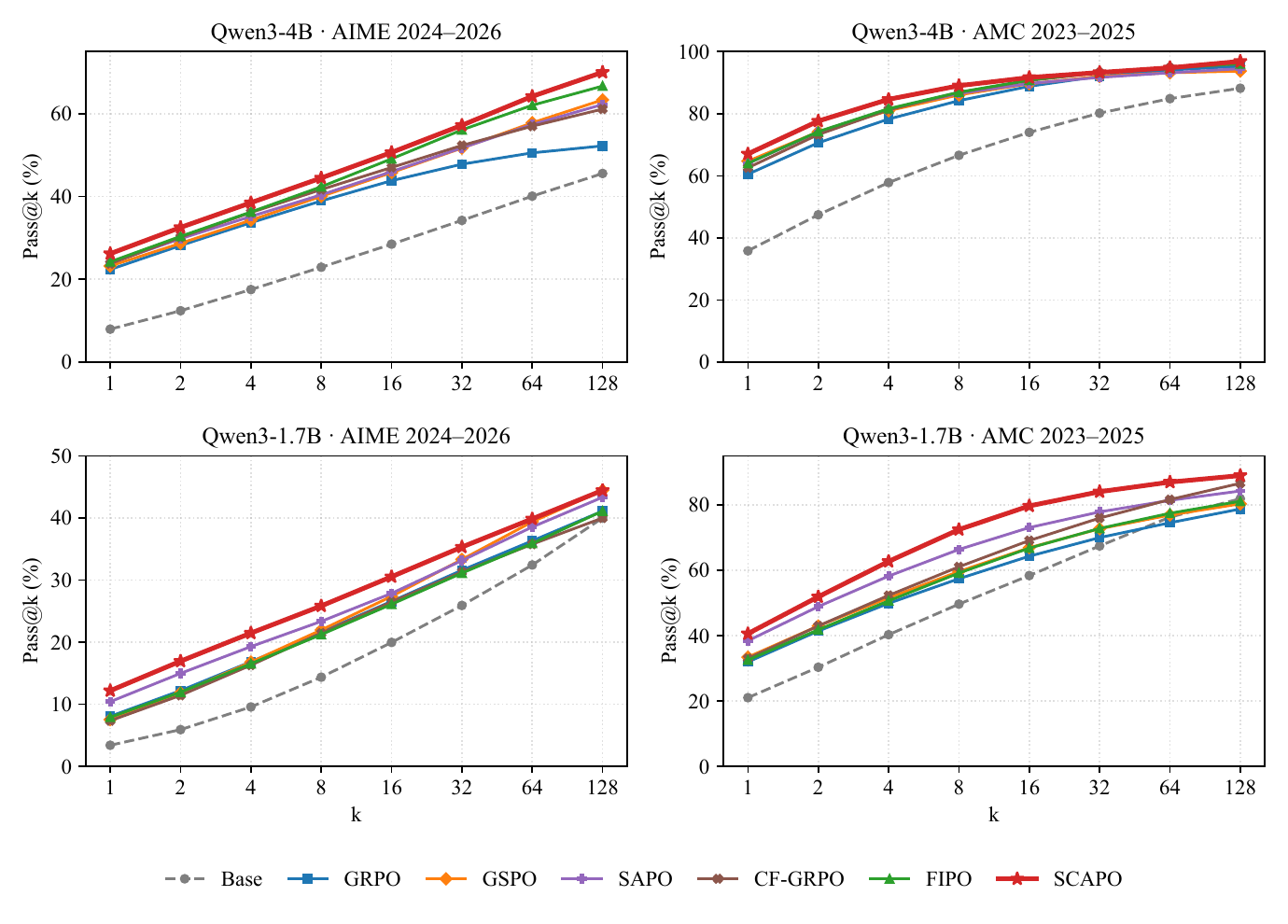}
\caption{Pass@$k$ on the AIME 2024--2026 and AMC 2023--2025 benchmarks at both model scales.}
\label{fig:appendix-passk}
\end{figure}

\end{document}